\documentclass[11pt]{amsart}

\usepackage{amsmath,amssymb,amsthm,amsrefs,a4wide}
\usepackage{latexsym}
\usepackage{indentfirst}
\usepackage{graphicx}
\usepackage{placeins}
\usepackage{booktabs}
\usepackage{algorithm}
\usepackage{algorithmic}
\usepackage{multirow}
\usepackage{color}

\theoremstyle{plain}
\newtheorem{theorem}{Theorem}

\theoremstyle{definition}
\newtheorem{defn}[theorem]{Definition}
\newtheorem{example}{Example}

\theoremstyle{remark}
\newtheorem{remark}{Remark}

\DeclareMathOperator{\argmin}{argmin}

\DeclareMathOperator{\diag}{diag}

\DeclareMathOperator{\TV}{TV}

\newcommand{\eps}{\epsilon}

\newcommand{\bbm}{\begin{bmatrix}}
\newcommand{\ebm}{\end{bmatrix}}
\newcommand{\R}{\mathbb{R}}

\renewcommand{\P}{\mathcal{P}}
\newcommand{\F}{\mathcal{F}}

\newcommand{\grad}{\nabla}

\begin{document}

\title[Tangent differential privacy]{Tangent differential privacy}

\author[]{Lexing Ying}
\address[Lexing Ying]{Department of Mathematics, Stanford University,  Stanford, CA 94305} \email{lexing@stanford.edu}

\thanks{The author thanks Yiping Lu for constructive discussions.}



\begin{abstract}
  Differential privacy is a framework for protecting the identity of individual data points in the decision-making process. In this note, we propose a new form of differential privacy called tangent differential privacy. Compared with the usual differential privacy that is defined uniformly across data distributions, tangent differential privacy is tailored towards a specific data distribution of interest. It also allows for general distribution distances such as total variation distance and Wasserstein distance. In the case of risk minimization, we show that entropic regularization guarantees tangent differential privacy under rather general conditions on the risk function.
\end{abstract}

\maketitle

\section{Introduction}\label{sec:intro}

Differential privacy is a framework for protecting the identity of individual data points in the machine learning process. The most commonly discussed differential privacy is $\eps$-differential privacy. A randomized algorithm is called $\eps$-differential private if, for any two input data distributions that differ by one element, the ratio of the probabilities at any outcome is bounded by at most $\exp(\eps)$. The definition clearly shows that differential privacy is a uniform concept across all data distributions. In many machine learning applications, one often cares about a specific data distribution and asks privacy questions about when a single or a small number of data points is deleted from or added to this specific data distribution.

To address such questions, we propose here {\em tangent differential privacy}, which is a privacy concept tailored towards a specific data distribution. When applying to the case of risk minimization (such as supervised learning), we show that entropic regularization guarantees tangent differential privacy under rather general conditions.

{\bf Related work.} The concepts of $\eps$-differential privacy and $(\eps,\delta)$-differential privacy were first proposed in \cite{dwork2006calibrating,dwork2006our} and a wonderful resource for this vast literature is \cite{dwork2014algorithmic}. Several efforts have been devoted to relax or reformulate differential privacy, with examples including Renyi differential privacy \cite{mironov2017renyi}, concentrated differential privacy \cite{bun2016concentrated,dwork2016concentrated}, and Lipschitz privacy \cite{koufogiannis2015optimality}. In a broader context, other related forms of privacy concepts have also been developed, such as local differential privacy \cite{evfimievski2003limiting,kasiviswanathan2011can,duchi2013local} and the recently proposed metric privacy \cite{boedihardjo2024metric,boedihardjo2024private}. The concept of tangent differential privacy proposed here is closely related to Lipschitz privacy, though the latter is defined as a uniform concept across all data distributions.

{\bf Contents.} The rest of the note is organized as follows. Section 2 introduces the concept of tangent differential privacy. Section 3 considers the risk minimization problem and proposes entropic regularization as a solution of tangent differential privacy for both total variation and Wasserstein distances. Section 4 concludes with some discussions.

\section{Tangent differential privacy}\label{sec:tp}

Let $X$ be the metric space of the data points, and $\P(X)$ be the space of distributions over $X$. Let $W$ be the metric space of outputs, and $\P(W)$ be the space of distributions over $W$. We can think of $W$ as $\R$, the space of regression functions, or the space of neural network weights \cite{abadi2016deep}. To discuss differential privacy, let $A$ be a randomized algorithm that takes $p \in \P(X)$ and produces a randomized output $w$. Because $A$ is random, we can regard it as a (typically nonlinear) map
\[
A: \P(X) \rightarrow \P(W),
\]
taking $p(x)$ to a distribution $q(w)$. When $q(w)$ has a bounded density, we can also consider $\log\circ A$
\[
\log\circ A: \P(X) \rightarrow \F(W),
\]
taking $p(x)$ to a function $(\log q)(w)$.

Let us denote $T_p$ and $T_q$ as the tangent spaces of signed measures at $p$ and $q$, respectively. The tangent map of $A$ at $p$ is $DA_p: T_p \rightarrow T_q$. Suppose that $p$ is the data distribution of interest. For any $p'$ close to $p$, the linear approximation suggests that
\begin{equation}\label{eq:Apl}
  Ap'-Ap \approx DA_p \cdot (p'-p).
\end{equation}
In the usual setting, $p$ can be an empirical distribution with $N$ data samples $\{x_i\}$ and $p'$ is obtained by removing a distinguished sample $x_k$:
\[
p(x) = \frac{1}{N} \sum_{i=1}^N \delta_{x_i}(x), \quad
p'(x) = \frac{1}{N-1} \sum_{i\not=k} \delta_{x_i}(x).
\]
This also extends naturally to the situation where $p'$ is obtained from $p$ by changing a small number of data points. Similarly, if $T_{\log q}$ is the tangent space at $\log q$, the tangent map of $\log\circ A$ at $p$ is $D(\log\circ A)_p : T_p \rightarrow T_{\log q}$. For any $p'$ close to $p$, we have
\begin{equation}\label{eq:logApl}
  \log(Ap')-\log(Ap) \approx D(\log\circ A)_p \cdot (p'-p).
\end{equation}

When $\P(X)$ and $\P(W)$ are endowed with distances, they induce corresponding norms on $T_p$, $T_q$, and $T_{\log q}$. Here are two common examples.
\begin{example}
  Consider the total variation distance on $\P(X)$ and $\P(W)$, i.e.,
  \[
  d_{\TV}(p,p') = 2\cdot \max_{S\subset X} |p(S)-p'(S)|, \quad
  d_{\TV}(q,q') = 2\cdot \max_{S\subset W} |q(S)-q'(S)|.
  \]
  Here, we introduce an extra factor of $2$ to ensure that they are consistent with the $L^1$ norms when $X$ and $W$ are countable.  This setup results in the TV norm for $T_p$, TV norm for $T_q$, and $L^\infty$ norm for $T_{\log q}$.
\end{example}

\begin{example}
  Consider the Wasserstein-2 distance on $\P(X)$ and still the total variation distance on $\P(W)$. This results in the following weighted Sobolev norm for $T_p$:
  \[
  \|\eps\|_{\dot{H}^{-1}(p)}^2 := \min_{f:\grad\cdot(fp) = \eps} \int |f(x)|^2 p(x) dx
  \]
  for $\eps\in T_p$ \cite{peyre2018comparison}. For $q$, we still have TV norm for $T_q$ and $L^\infty$ norm for $T_{\log q}$.
\end{example}

With these preparations, we are ready to introduce the following definitions.

\begin{defn}
  $A$ is differentiable at $p$ for the norm pair $(\|\cdot\|_{T_p}, \|\cdot\|_{T_q})$ with bound $C_p$ if
  \begin{equation}\label{eq:diff}
    \|DA_p\|_{\|\cdot\|_{T_p} \rightarrow \|\cdot\|_{T_q}} \le C_p.
  \end{equation}
\end{defn}
Consider the case of $\|\cdot\|_{T_q}$ equal to the TV norm as in the previous examples. For $p'$ with $\|p'-p\|$ small, using \eqref{eq:Apl} and \eqref{eq:diff} leads to
\[
\| Ap- Ap' \|_{\TV} \le C_p \|p'-p\| + o(\|p'-p\|),
\]
i.e., for any set $S\subset W$,
\[
2\cdot |(Ap)(S)-(Ap')(S)| \le C_p \|p'-p\| + o(\|p'-p\|).
\]

\begin{defn}
  $A$ satisfies tangent differential privacy at $p$ for the norm pair $(\|\cdot\|_{T_p}, \|\cdot\|_{T_{\log q}})$ with bound $C_p$ if
  \begin{equation}\label{eq:tdp}
    \|D(\log\circ A)_p\|_{\|\cdot\|_{T_p} \rightarrow \|\cdot\|_{T_{\log q}}} \le C_p.
  \end{equation}
\end{defn}
Consider the case of $\|\cdot\|_{T_{\log q}}$ equal to the $L^\infty$ norm as above. For $p'$ with $\|p'-p\|$ small, using \eqref{eq:logApl} and \eqref{eq:tdp} leads to
\[
\| \log(Ap)- \log(Ap') \|_{L^\infty}  = \left\| \log\left(\frac{Ap}{Ap'}\right) \right\|_{L^\infty}  \le C_p \|p'-p\| + o(\|p'-p\|),
\]
i.e., for any $w\in W$,
\[
\exp( - C_p \|p'-p\| + o(\|p'-p\|)) \le \frac{(Ap)(w)}{(Ap')(w)} \le \exp( C_p \|p'-p\| + o(\|p'-p\|) ).
\]
Therefore, for any set $S\subset W$,
\[
\exp( - C_p \|p'-p\| + o(\|p'-p\|)) \le \frac{\int_S (Ap)(w)dw}{\int_S (Ap')(w)dw} \le \exp( C_p \|p'-p\| + o(\|p'-p\|) ).
\]
This is a more quantitative version of differential privacy adapted to the data distribution $p$.

\begin{remark}
  (a) By working directly with the space of distributions $\P(X)$, the concept of tangent differential privacy is defined without direct reference to the number of data samples in the distribution. Therefore, it allows for changing either a single data sample or a small fraction of samples.

  (b) Working with different distances on $\P(X)$ leads to different types of privacy considerations. For example, the total variation distance on $\P(X)$ corresponds to the Hamming distance case of the $\eps$-differential privacy.  The Wasserstein distance case is related to the metric privacy setup.
\end{remark}

\section{Risk minimization}\label{sec:rm}

We consider the case where the output in $W$ is obtained via an optimization procedure, for example, the empirical risk minimization. Given the data distribution $p(x)$ and risk function $r(w,x)\ge 0$, the goal is
\[
\min_w \int r(w,x) p(x) dx.
\]
The solution $w$ of this minimization problem depends deterministically on $p(x)$. In order to discuss differential privacy, one needs to consider a randomized algorithm with the output distributed over $W$. Here, we propose to adopt entropic regularization following \cite{mcsherry2007mechanism} and seek $q(w)\in \P(W)$
\[
q = \argmin_{q\in\P(W)} \int q(w) \left( \int r(w,x) p(x) dx \right) dw + \beta^{-1} \int q(w) \ln q(w) dw.
\]
The solution is the Gibbs distribution
\begin{equation}\label{eq:qw}
q(w) = \frac{\exp(-\beta \int r(w,x) p(x) dx)}{\int_W \exp(-\beta \int r(w',x) p(x) dx) dw'},
\end{equation}
or simply written as $q(w) \propto \exp(-\beta \int r(w,x) p(x) dx)$. Then, the map $A:\P(X) \rightarrow \P(W)$ takes from $p(x)$ to $q(w)$.

\begin{remark}
  The distribution \eqref{eq:qw} can, in principle, be sampled using Monte Carlo methods, such as Langevin dynamics. One popular differentially private algorithm is noisy-SGD \cite{bassily2014private,bassily2019private}, and there is a close connection between noisy-SGD and Langevin dynamics \cite{welling2011bayesian,cheng2020stochastic}.
\end{remark}

Fixing $p(x)$, let us compute the differential $DA_p: T_p \rightarrow T_q$. Its kernel as a function of $(w,x)$ is given by
\[
-\beta \int \left(q(w) \delta(w-w') - q(w) q(w')\right) r(w',x) dw'.
\]
When $X$ and $W$ are finite sets, this can be written in the matrix form as 
\[
-\beta (\diag(q) - qq^\top) r,
\]
where here $r$ denotes a matrix with value $r(w,x)$ at entry $(w,x)$.

The differential of $D(\log\circ A)_p: T_p \rightarrow T_{\log q}$ can also be computed easily with the chain rule. Its kernel as a function of $(w,x)$ is
\[
-\beta \int \left(\delta(w-w') - q(w')\right) r(w',x) dw'.
\]
Again, when $X$ and $W$ are finite sets, the matrix form is 
\[
-\beta (I - \mathbf{1} q^\top) r,
\]
where $\mathbf{1}$ stands for the all one column vector. Below, we show that the entropic regularization guarantees tangent differential privacy under rather general conditions for both the TV distance and the Wasserstein distance on $\P(X)$.

\subsection{Total variation distance on $\P(X)$}

Recall from Example 1 that one has TV norm for $\|\cdot\|_{T_p}$, TV norm for $\|\cdot\|_{T_q}$, and $L^\infty$ norm for $\|\cdot\|_{T_{\log q}}$.

\begin{theorem}
  If $\max_x \int_W q(w) r(w,x) dw \le R$, then $A$ is differentiable at $p$ for the norm pair $(\TV,\TV)$ with bound $2\beta R$.
\end{theorem}

\begin{proof}
  Pick any signed measure $\eps(x) \in T_p$. Up to the $-\beta$ factor, $(DA_p \eps)(w)$ is equal to
  \begin{align*}
        & q(w) \iint \left(\delta(w-w')-q(w')\right) r(w',x) \eps(x) dx dw'\\
    =\; & q(w) \int r(w,x) \eps(x) dx - q(w) \iint q(w')r(w',x) \eps(x) dx dw' 
  \end{align*}
  Among the two terms, the TV norm of the first term is bounded by
  \[
  \int_W \left| \int_X q(w) r(w,x) \eps(x) dx \right| dw \le \left(\max_x \int_W |q(w)r(w,x)|dw\right) \cdot \|\eps\|_{\TV} \le
  R \|\eps\|_{\TV},
  \]
  where we use the non-negativity of $q(w) r(w,x)$. The same estimate applies to the second term.  Putting together shows that $\|DA_p\|_{\TV\rightarrow \TV} \le 2\beta R$.
\end{proof}

\begin{remark}
  (a) The product $\beta R$ controls the sensitivity of $A$. One can achieve this by adopting either small $R$ (safer risk function) or small $\beta$ (stronger entropic regularization).

  (b) One way to ensure $\max_x \int_W q(w) r(w,x) dw \le R$ is $\max_{w,x}|r(w,x)|\le R$. But this can be strict as it does not take into consideration the distributions $p(x)$ and $q(w)$.
  
  (c) The quantity $\max_{x} \int_W q(w) r(w,x) dw$ can be estimated. Suppose that we have an algorithm that can sample $w\sim q(w)$. First, for each $w$, iterate over the data point $x$ and accumulate $r(w,x)$ for each $x$. Second, for each data point $x$, dividing the accumulated value by the number of $w$ gives the estimate of $\int_W q(w) r(w,x) dw$ for $x$. Finally, taking the maximum of these estimates over $x$ gives the approximation to $\max_{x} \int_W q(w) r(w,x) dw$.
\end{remark}

\begin{theorem}
  If $\max_{x,w} |r(w,x)| \le R$, then $A$ satisfies tangent differential privacy at $p$ for the norm pair $(\TV, L^\infty)$ with bound $2\beta R$.
\end{theorem}

\begin{proof}
  Pick any $\eps \in T_p$. Up to the $-\beta$ factor, $(D(\log\circ A)_p \eps)(w)$ at each $w$ is 
  \[
  \iint \left(\delta(w-w')-q(w')\right) r(w',x) \eps(x) dx dw' =  \int r(w,x) \eps(x) dx - \iint q(w') r(w',x) \eps(x) dx dw'.
  \]
  Among the two terms, the first one is bounded at $w$ with
  \[
  \left|\int r(w,x) \eps(x) dx \right| \le R \|\eps\|_{\TV}.
  \]
  The second one can be bounded in the same way. Therefore, $\|D(\log\circ A)_p\|_{\TV\rightarrow L^\infty} \le 2\beta R$.
\end{proof}

\begin{remark}
  Examples of bounded $r(w,x)$ include the Savage loss, the tangent loss, and the $0/1$ loss. Using any of these losses automatically guarantees tangent differential privacy for $(\TV, L^\infty)$.
\end{remark}

\subsection{Wasserstein distance on $\P(X)$}

Recall from Example 2 that we have the $\dot{H}^{-1}(p)$ norm for $\|\cdot\|_{T_p}$, TV norm for $\|\cdot\|_{T_q}$, and $L^\infty$ norm for $\|\cdot\|_{T_{\log q}}$. Recall that the $\dot{H}^{-1}(p)$ norm and its dual norm are given by
\[
\|\eps\|_{\dot{H}^{-1}(p)}^2 = \min_{f:\grad\cdot(fp) = \eps} \int |f(x)|^2 p(x) dx, \quad
\|g\|_{\dot{H}^1(p)}^2 =  \int |\grad g(x)|^2 p(x) dx.
\]

\begin{theorem}
  If $\| \int_W q(w)r(w,\cdot) dw\|_{\dot{H}^1(p)} \le R$, then $A$ is differentiable at $p$ for the norm pair $(\dot{H}^{-1}(p), \TV)$ with bound $2\beta R$.
\end{theorem}

\begin{proof}
  Pick any $\eps(x) \in \dot{H}^1(p)$. Up to the $-\beta$ factor, $(DA_p \eps)(w)$ is equal to
  \[
  q(w) \int r(w,x) \eps(x) dx - q(w) \iint q(w')r(w',x) \eps(x) dx dw'.
  \]
  The TV norm of the first term can be bounded by
  \begin{align*}
    \int_W \left| \int_X q(w) r(w,x) \eps(x) dx \right| dw \le
    \left\| \int |q(w)r(w,\cdot)| dw \right\|_{\dot{H}^1(p)} \|\eps\|_{\dot{H}^{-1}(p)} \le R\|\eps\|_{\dot{H}^{-1}(p)},
  \end{align*}
  where we use the non-negativity of $q(w) r(w,\cdot)$. The same estimate can bound the second term. Therefore, $\|DA_p\|_{ \dot{H}^{-1}(p) \rightarrow L^1} \le 2\beta R$.
\end{proof}


\begin{remark}
  (a) One way to ensure $\| \int_W q(w)r(w,\cdot) dw\|_{\dot{H}^1(p)} \le R$ is $\max_w \|r(w,\cdot)\|_{\dot{H}^1(p)} \le R$. However, this can be too strict as it does not take into consideration the distributions $p(x)$ and $q(w)$.
  
  (b) The quantity $\| \int_W q(w)r(w,\cdot) dw\|_{\dot{H}^1(p)} = \left( |\int_W q(w) \grad_x r(w,x) dw |^2 p(x) dx \right)^{1/2}$ can be estimated instead. Suppose that we have an algorithm that samples $w\sim q(w)$. First, for each $w$, iterate over $x$ and accumulate $\grad_x r(w,x)$ for each $x$. Second, for each $x$, divide the accumulated value by the number of $w$ to get an estimate of $|\int_W q(w) \grad_x r(w,x) dw |$. Its square is an estimate for $|\int_W q(w) \grad_x r(w,x) dw |^2$. Finally, averaging the squares over $x$ and taking the square root gives an approximation to $\| \int_W q(w)r(w,\cdot) dw\|_{\dot{H}^1(p)}$.
\end{remark}

\begin{theorem}
  If $\max_{w} \|r(w,\cdot)\|_{\dot{H}^1(p)} \le R$, then $A$ satisfies tangent differential privacy at $p$ for the norm pair $(\dot{H}^{-1}(p), L^\infty)$ with bound $2\beta R$.
\end{theorem}

\begin{proof}
  Pick any $\eps \in L^1$. Up to the $-\beta$ factor, $(D(\log\circ A)_p \eps)(w)$ at each $w$ is 
  \[
  \iint \left(\delta(w-w')-q(w')\right) r(w',x) \eps(x) dx dw' =  \int r(w,x) \eps(x) dx - \iint q(w') r(w',x) \eps(x) dx dw'.
  \]
  The first term is bounded at $w$ by
  \[
  \int r(w,x) \eps(x) dx \le \|r(w,\cdot)\|_{\dot{H}^1(p)}\|\eps\|_{\dot{H}^{-1}(p)} \le R\|\eps\|_{\dot{H}^{-1}(p)}.
  \]
  The second term can be bounded in the same way. Therefore, $\|D(\log\circ A)_p\|_{\dot{H}^{-1}(p) \rightarrow L^\infty} \le 2\beta R$.
\end{proof}

\section{Discussion}\label{sec:disc}

In this note, we propose tangent differential privacy as a new form of differential privacy. Compared with the usual differential privacy that is defined uniformly across data distributions, tangent differential privacy is tailored towards a specific data distribution of interest. For empirical risk minimization of supervised learning, entropic regularization guarantees tangent differential privacy under rather general conditions on the risk function. Some directions for future work include
\begin{itemize}
\item Extend the framework to unsupervised learning and online learning problems;
\item Explore alternatives or approximations to \eqref{eq:qw} since sampling the Gibbs distribution $q(w)$ can be challenging when it exhibits meta-stability.
\end{itemize}

\bibliographystyle{abbrv}

\bibliography{ref}

\begin{bibdiv}
\begin{biblist}

\bib{abadi2016deep}{inproceedings}{
      author={Abadi, Martin},
      author={Chu, Andy},
      author={Goodfellow, Ian},
      author={McMahan, H~Brendan},
      author={Mironov, Ilya},
      author={Talwar, Kunal},
      author={Zhang, Li},
       title={Deep learning with differential privacy},
        date={2016},
   booktitle={Proceedings of the 2016 acm sigsac conference on computer and
  communications security},
       pages={308\ndash 318},
}

\bib{bassily2019private}{article}{
      author={Bassily, Raef},
      author={Feldman, Vitaly},
      author={Talwar, Kunal},
      author={Guha~Thakurta, Abhradeep},
       title={Private stochastic convex optimization with optimal rates},
        date={2019},
     journal={Advances in neural information processing systems},
      volume={32},
}

\bib{bassily2014private}{inproceedings}{
      author={Bassily, Raef},
      author={Smith, Adam},
      author={Thakurta, Abhradeep},
       title={Private empirical risk minimization: Efficient algorithms and
  tight error bounds},
organization={IEEE},
        date={2014},
   booktitle={2014 ieee 55th annual symposium on foundations of computer
  science},
       pages={464\ndash 473},
}

\bib{boedihardjo2024metric}{article}{
      author={Boedihardjo, March},
      author={Strohmer, Thomas},
      author={Vershynin, Roman},
       title={Metric geometry of the privacy-utility tradeoff},
        date={2024},
     journal={arXiv preprint arXiv:2405.00329},
}

\bib{boedihardjo2024private}{article}{
      author={Boedihardjo, March},
      author={Strohmer, Thomas},
      author={Vershynin, Roman},
       title={Private measures, random walks, and synthetic data},
        date={2024},
     journal={Probability Theory and Related Fields},
       pages={1\ndash 43},
}

\bib{bun2016concentrated}{inproceedings}{
      author={Bun, Mark},
      author={Steinke, Thomas},
       title={Concentrated differential privacy: Simplifications, extensions,
  and lower bounds},
organization={Springer},
        date={2016},
   booktitle={Theory of cryptography conference},
       pages={635\ndash 658},
}

\bib{cheng2020stochastic}{inproceedings}{
      author={Cheng, Xiang},
      author={Yin, Dong},
      author={Bartlett, Peter},
      author={Jordan, Michael},
       title={Stochastic gradient and langevin processes},
organization={PMLR},
        date={2020},
   booktitle={International conference on machine learning},
       pages={1810\ndash 1819},
}

\bib{duchi2013local}{inproceedings}{
      author={Duchi, John~C},
      author={Jordan, Michael~I},
      author={Wainwright, Martin~J},
       title={Local privacy and statistical minimax rates},
organization={IEEE},
        date={2013},
   booktitle={2013 ieee 54th annual symposium on foundations of computer
  science},
       pages={429\ndash 438},
}

\bib{dwork2006our}{inproceedings}{
      author={Dwork, Cynthia},
      author={Kenthapadi, Krishnaram},
      author={McSherry, Frank},
      author={Mironov, Ilya},
      author={Naor, Moni},
       title={Our data, ourselves: Privacy via distributed noise generation},
organization={Springer},
        date={2006},
   booktitle={Advances in cryptology-eurocrypt 2006: 24th annual international
  conference on the theory and applications of cryptographic techniques, st.
  petersburg, russia, may 28-june 1, 2006. proceedings 25},
       pages={486\ndash 503},
}

\bib{dwork2006calibrating}{inproceedings}{
      author={Dwork, Cynthia},
      author={McSherry, Frank},
      author={Nissim, Kobbi},
      author={Smith, Adam},
       title={Calibrating noise to sensitivity in private data analysis},
organization={Springer},
        date={2006},
   booktitle={Theory of cryptography: Third theory of cryptography conference,
  tcc 2006, new york, ny, usa, march 4-7, 2006. proceedings 3},
       pages={265\ndash 284},
}

\bib{dwork2014algorithmic}{article}{
      author={Dwork, Cynthia},
      author={Roth, Aaron},
      author={others},
       title={The algorithmic foundations of differential privacy},
        date={2014},
     journal={Foundations and Trends{\textregistered} in Theoretical Computer
  Science},
      volume={9},
      number={3--4},
       pages={211\ndash 407},
}

\bib{dwork2016concentrated}{article}{
      author={Dwork, Cynthia},
      author={Rothblum, Guy~N},
       title={Concentrated differential privacy},
        date={2016},
     journal={arXiv preprint arXiv:1603.01887},
}

\bib{evfimievski2003limiting}{inproceedings}{
      author={Evfimievski, Alexandre},
      author={Gehrke, Johannes},
      author={Srikant, Ramakrishnan},
       title={Limiting privacy breaches in privacy preserving data mining},
        date={2003},
   booktitle={Proceedings of the twenty-second acm sigmod-sigact-sigart
  symposium on principles of database systems},
       pages={211\ndash 222},
}

\bib{kasiviswanathan2011can}{article}{
      author={Kasiviswanathan, Shiva~Prasad},
      author={Lee, Homin~K},
      author={Nissim, Kobbi},
      author={Raskhodnikova, Sofya},
      author={Smith, Adam},
       title={What can we learn privately?},
        date={2011},
     journal={SIAM Journal on Computing},
      volume={40},
      number={3},
       pages={793\ndash 826},
}

\bib{koufogiannis2015optimality}{article}{
      author={Koufogiannis, Fragkiskos},
      author={Han, Shuo},
      author={Pappas, George~J},
       title={Optimality of the laplace mechanism in differential privacy},
        date={2015},
     journal={arXiv preprint arXiv:1504.00065},
      volume={10},
}

\bib{mcsherry2007mechanism}{inproceedings}{
      author={McSherry, Frank},
      author={Talwar, Kunal},
       title={Mechanism design via differential privacy},
organization={IEEE},
        date={2007},
   booktitle={48th annual ieee symposium on foundations of computer science
  (focs'07)},
       pages={94\ndash 103},
}

\bib{mironov2017renyi}{inproceedings}{
      author={Mironov, Ilya},
       title={R{\'e}nyi differential privacy},
organization={IEEE},
        date={2017},
   booktitle={2017 ieee 30th computer security foundations symposium (csf)},
       pages={263\ndash 275},
}

\bib{peyre2018comparison}{article}{
      author={Peyre, R{\'e}mi},
       title={Comparison between w2 distance and h norm, and localization of
  wasserstein distance},
        date={2018},
     journal={ESAIM: Control, Optimisation and Calculus of Variations},
      volume={24},
      number={4},
       pages={1489\ndash 1501},
}

\bib{welling2011bayesian}{inproceedings}{
      author={Welling, Max},
      author={Teh, Yee~W},
       title={Bayesian learning via stochastic gradient langevin dynamics},
organization={Citeseer},
        date={2011},
   booktitle={Proceedings of the 28th international conference on machine
  learning (icml-11)},
       pages={681\ndash 688},
}

\end{biblist}
\end{bibdiv}

\end{document}